\documentclass{article}
\usepackage{ijcai21}

\usepackage{times}

\usepackage{soul}
\usepackage{url}
\usepackage[hidelinks]{hyperref}
\usepackage[utf8x]{inputenc}
\usepackage[small]{caption}
\usepackage{graphicx}
\usepackage{amsmath}
\usepackage{booktabs}
\usepackage{amsmath,amsthm,xspace}
\usepackage{amssymb}
\usepackage{xspace}
\usepackage[noresetcount,ruled,noend]{algorithm2e}
\usepackage{multirow}
\usepackage{longtable}

\usepackage[colorinlistoftodos]{todonotes}

\newtheorem{theorem}{Theorem}

\newtheorem{example}{Example}
\newtheorem{property}{Property}

\SetKwFunction{acycchecker}{acycChecker}

\newcommand{\calC}{\ensuremath{\mathcal{C}}\xspace}

\newcommand{\xhat}{\ensuremath{\hat{x}}\xspace}
\newcommand{\yhat}{\ensuremath{\hat{y}}\xspace}

\newcommand{\allvars}{V\xspace}
\newcommand{\currentclusters}{\calC}
\newcommand{\ineq}[1]{cons(#1)}
\newcommand{\varsof}[1]{varsof(#1)}
\newcommand{\rccluster}{RC-cluster\xspace}
\newcommand{\rcclusters}{RC-clusters\xspace}
\newcommand{\rcD}[3]{\ensuremath{{#1}^{rc}_{#2, #3}}\xspace}
\newcommand{\minrc}[1]{\ensuremath{minrc(#1)}}

\newcommand{\gob}{GOBNILP\xspace}
\newcommand{\cpb}{CPBayes\xspace}
\newcommand{\cpbcut}{ELSA\xspace}
\newcommand{\scip}{SCIP\xspace}

\newcommand{\cplex}{cplex\xspace}

\newcommand{\tuple}[1]{\ensuremath{\left\langle #1 \right\rangle}}

\title{Improved Acyclicity Reasoning for Bayesian Network Structure Learning with Constraint Programming}

\author{
Fulya Trösser$^1$\footnote{Contact Author}\and
Simon de Givry$^1$\And
George Katsirelos$^2$\\
\affiliations
$^1$Université de Toulouse, INRAE, UR MIAT, F-31320, Castanet-Tolosan, France\\
$^2$UMR MIA-Paris, INRAE, AgroParisTech, Univ. Paris-Saclay, 75005 Paris, France\\
\emails
\{fulya.ural, simon.de-givry\}@inrae.fr,
gkatsi@gmail.com
}

\begin{document}

\maketitle

\begin{abstract}

Bayesian networks are probabilistic graphical models with a wide range of application areas including gene regulatory networks inference, risk analysis and image processing. Learning the structure of a Bayesian network (BNSL) from discrete data is known to be an NP-hard task with a superexponential search space of directed acyclic graphs. 
%There have been several successful approaches such as ILP, A* search, depth-first branch-and-bound (BnB) search, and breadth-first BnB search, each with different pros and cons.
In this work, we propose a new polynomial time algorithm for discovering 
a subset of all possible
cluster cuts, a greedy algorithm for 
approximately solving the resulting linear program, and
a generalised arc consistency algorithm for the acyclicity constraint.
We embed these in the constraint programming-based branch-and-bound solver CPBayes 
and show that, despite being suboptimal, they improve performance by orders of
magnitude. The resulting solver also compares favourably with GOBNILP, a state-of-the-art solver for the BNSL problem which solves an NP-hard problem to discover each cut
and solves the linear program exactly.

\end{abstract}

\section{Introduction}\label{Intro}

Towards the goal of explainable AI, Bayesian networks offer a rich framework for probabilistic reasoning. Bayesian Network Structure Learning (BNSL) from discrete observations corresponds to finding a compact model which best explains the data. It defines an NP-hard problem with a superexponential search space of Directed Acyclic Graphs (DAG). Several constraint-based (exploiting local conditional independence tests) and score-based (exploiting a global objective formulation) BNSL methods have been developed in the past.

Complete methods for score-based BNSL include dynamic programming~\cite{SilanderUAI2006}, heuristic search \cite{YuanMaloneJAIR2013,FanYuanAAAI2015},
maximum satisfiability \cite{berg2014learning},
branch-and-cut~\cite{CussensAIJ2017} 
and constraint programming~\cite{cpbayes}. Here, we focus on the latter two. 

\gob \cite{CussensAIJ2017} is a state-of-the-art solver for BNSL.
It implements branch-and-cut in an integer linear programming (ILP)
solver. At each node of the branch-and-bound tree, it generates cuts
that improve the linear relaxation. A major class of cuts generated by
GOBNILP are \emph{cluster cuts}, which identify sets of parent sets
that cannot be used together in an acyclic graph. In order to find
cluster cuts, \gob solves an NP-hard subproblem created from the
current optimal solution of the linear relaxation.

\cpb \cite{cpbayes} is a constraint programming-based (CP) method for
BNSL. It uses a CP model that exploits symmetry and dominance
relations present in the problem, subproblem caching, and a pattern
database to compute lower bounds, adapted from heuristic search
\cite{FanYuanAAAI2015}. \citeauthor{cpbayes} showed that \cpb is competitive with
\gob in many instances. In contrast to \gob, the inference mechanisms
of \cpb are very lightweight, which allows it to explore many orders
of magnitude more nodes 
per time unit, 
even accounting for the fact that computing the
pattern databases before search can sometimes consume considerable
time. On the other hand, the lightweight pattern-based bounding
mechanism can take into consideration only limited information about
the current state of the search. Specifically, it can take into
account the current total ordering implied by the DAG under
construction, but no information that has been derived about the
potential parent sets of each vertex, i.e., the current domains of
parent set variables.

In this work, we derive a lower bound that is computationally
cheaper than that computed by \gob. We give in Section~\ref{sec:Cluster}
a polynomial-time algorithm that discovers a class of cluster
cuts that provably improve the linear relaxation. 
In Section~\ref{sec:lpsolve}, we give a greedy algorithm for solving the linear
relaxation, 
inspired by similar algorithms for MaxSAT and Weighted Constraint Satisfaction Problems (WCSP). 
Finally, in Section~\ref{sec:GAC} we give an
algorithm that enforces generalised arc consistency on the acyclicity
constraint, based on previous work by \citeauthor{cpbayes}, but with
improved complexity and practical performance. In Section~\ref{sec:exp}, we
show that our implementation of these techniques in \cpb leads to
significantly improved performance, both in the size of the search
tree explored and in runtime.

\section{Preliminaries}\label{sec:Prel}

We give here only minimal background on (integer) linear
programming and constraint programming, and refer the reader to
existing literature \cite{papadimitriou1998combinatorial,cphandbook}
for more.

\subsubsection{Constraint Programming}

A constraint satisfaction problem (CSP) is a tuple
$\tuple{V, D, C}$, where $V$ is a set of variables, $D$ is a
function mapping variables to domains and C is a set of
constraints. An assignment $A$ to $V' \subseteq V$ is a mapping from
each $v \in V'$ to $D(v)$. A complete assignment is an assignment to
$V$. If an assignment maps $v$ to $a$, we say it assigns $v=a$.
A constraint is a pair $\tuple{S, P}$, where $S \subseteq V$ is
the \emph{scope} of the constraint and $P$ is a predicate over
$\prod_{V \in S}D(V)$ which accepts assignments to $S$ that
\emph{satisfy} the constraint. For an assignment $A$ to
$S' \supseteq S$, let $A'|_S$ be the restriction of $A$ to
$S$. We say that $A$ satisfies $c = \tuple{S, P}$ if $A|_S$ satisfies
$c$. A problem is satisfied by $A$ if $A$
satisfies all constraints.

For a constraint $c = \tuple{S,P}$ and for $v \in S, a \in D(v)$,
$v=a$ is generalized arc consistent (GAC) for $c$ if there exists an
assignment $A$ that assigns $v=a$ and satisfies $c$. If for all
$v \in S, a \in D(v)$, $v=a$ is GAC for $c$, then $c$ is GAC.
If all constraints are GAC, the problem is GAC.
A constraint is associated with an algorithm $f_c$, called
the propagator for $c$, that 
removes (or \emph{prunes}) values from the domains of
variables in $S$ that are not GAC.

CSPs are typically solved by backtracking search, using propagators to
reduce domains at each node and avoid parts of the search tree that
are proved to not contain any solutions. Although CSPs are decision
problems, the technology can be used to solve optimization problems
like BNSL by, for example, using branch-and-bound and embedding
the bounding part in a propagator. This is the approach used by \cpb.

\subsubsection{Integer Linear Programming}

A linear program (LP) is the problem of finding

\begin{eqnarray*}
  \min \{ c^T x \mid x \in \mathbb{R}^n \land Ax \geq b \land x \geq 0 \}
\end{eqnarray*}

where $c$ and $b$ are vectors, $A$ is a matrix, and $x$ is a vector of
variables. A feasible solution of this problem is one that satisfies
$x \in \mathbb{R}^n \land Ax \geq b \land x \geq 0$ and an optimal
solution is a feasible one that minimizes the objective function $c^T x$.
This can be found in polynomial
time. A row $A_i$ corresponds to an individual linear constraint and
a column $A^T_j$ to a variable. The dual of a linear program $P$ in the above
form is another linear program $D$:

\begin{eqnarray*}
  \max \{b^T y \mid y \in \mathbb{R}^m \land A^Ty \leq c \land y \geq 0 \}
\end{eqnarray*}

where $A, b, c$ are as before and $y$ is the vector of dual
variables. Rows of the dual correspond to variables of the primal and
vice versa. The objective value of any dual feasible solution is a
lower bound on the optimum of $P$. When $P$ is satisfiable, its dual
is also satisfiable and the values of their optima meet. For a given
feasible solution $\xhat$ of $P$, the slack of constraint $i$ is
$slack_{\xhat}(i) = A_i^T x - b_i$. Given a dual feasible solution
$\yhat$, $slack_{\yhat}^D(i)$ is the reduced cost of primal variable
$i, rc_{\yhat}(i)$.
% Complementary slackness states than in optimal
% solutions $\xstar, \ystar$, we have $\xstar_i \times rc_{\ystar}(i) = 0$, i.e., at
% least one of $\xstar_i$, $rc_{\ystar}(i)$ is 0.
The reduced cost
$rc_{\yhat}(i)$ is interpreted as a lower bound on the amount that the
dual objective would increase over $b^T\yhat$ if $x_i$ is forced to be
non-zero in the primal.

An integer linear program (ILP) is a linear program in which we replace the
constraint $x \in \mathbb{R}^n$ by $x \in \mathbb{Z}^n$ and it is an
NP-hard optimization problem. 

\subsubsection{Bayesian Networks}

A Bayesian network is a directed graphical model $B = \tuple{G, P}$
where $G=\tuple{V, E}$ is a directed acyclic graph (DAG) called the structure of
$B$ and $P$ are its parameters. A BN describes a normalised joint
probability distribution. Each vertex of the graph corresponds to a
random variable and presence of an edge between two vertices denotes
direct conditional dependence.
Each vertex $v_i$ is also associated with a Conditional Probability
Distribution $P(v_i \mid parents(v_i))$.
The CPDs are the parameters of $B$.
%
% Since a BN is acyclic, the global probability
% ditribution is normalized if the CPDs associated with each vertex are
% also normalized.

The approach which we use here for learning a BN from
data is the score-and-search method. Given
a set of multivariate discrete data $I = \{ I_1, \ldots, I_N \}$, a
scoring function $\sigma(G \mid I)$ measures the quality of the BN
with underlying structure $G$. The 
BNSL problem asks to find a structure $G$ that minimises
$\sigma(G \mid I)$ for some scoring function $\sigma$ and it is
NP-hard~\cite{Chickering1996}. Several scoring functions have been
proposed for this purpose, including BDeu \cite{BDeu-1,BDeu-2} and BIC
\cite{BIC-1,BIC-2}. These functions are decomposable and can be
expressed as the sum of local scores which only depend on the set of
parents (from now on, \emph{parent set}) of each vertex:
$\sigma_F(G \mid I) = \sum_{v \in V} \sigma_F^v(parents(v) \mid I)$
for $F \in \{BDeu, BIC\}$. In this setting, 
we first compute local scores and then
compute the structure of minimal score. Although there are potentially
an exponential number of local scores that have to be computed, the
number of parent sets actually considered
is often much smaller, for example because we restrict the
maximum cardinality of parent sets considered or we exploit dedicated pruning rules~\cite{CamposAAAI2010,CamposAIJ2018}.
We denote $PS(v)$
the set of candidate parent sets of $v$
and $PS^{-C}(v)$ those parent sets
that do not intersect $C$.
In the following, we
assume that local scores are precomputed and given as input, as is
common in similar works. We also omit explicitly mentioning $I$ or
$F$, as they are constant for solving any given instance.

Let $C$ be a set of vertices of a graph $G$. $C$ is a violated cluster
if the parent set of each vertex $v \in C$ intersects $C$. Then, we
can prove the following property:

\begin{property}
  \label{prop:cluster}
  A directed graph $G=\tuple{V, E}$ is acyclic if and only if it
  contains no violated clusters, i.e., for all $C \subseteq V$, there
  exists $v \in C$, such that $parents(v) \cap C = \emptyset$.
\end{property}

The \gob solver \cite{CussensAIJ2017} formulates the problem as
the following 0/1 ILP:

\begin{align}
  \min & \sum_{v \in \allvars, S \subseteq V \setminus \{v\}} \sigma^v(S) x_{v,S} \label{eq:ilp-obj} \\
  s.t. 
       & \sum_{S \in PS(v)} x_{v,S} = 1 & \forall v \in \allvars \label{eq:ilp-var} \\
       & \sum_{v \in C, S \in PS^{-C}(v)} x_{v,S} \geq 1 & \forall C \subseteq \allvars \label{eq:ilp-cluster} \\
       & x_{v, S} \in \{0,1\} & \forall v \in \allvars, S \in PS(v) \label{eq:ilp-integrality}
\end{align}

This ILP has a 0/1 variable $x_{v,S}$ for each candidate parent set $S$
of each vertex $v$ where $x_{v,S}=1$ means that $S$ is the parent set
of $v$. The objective \eqref{eq:ilp-obj} directly encodes the
decomposition of the scoring function. The constraint
\eqref{eq:ilp-var} asserts that exactly one parent set is selected for
each random variable. Finally, the \emph{cluster inequalities}
\eqref{eq:ilp-cluster} are violated when $C$ is a violated cluster.
We denote the cluster inequality for cluster $C$ as $\ineq{C}$ and the
0/1 variables involved as $\varsof{C}$.
As there is an exponential number of these, \gob
generates only those that improve the current linear
relaxation and they are referred to as \emph{cluster cuts}. This
itself is an NP-hard problem \cite{cussens2017bayesian}, which \gob also
encodes and solves as an ILP.
Interestingly, these inequalities are facets of the BNSL polytope
\cite{cussens2017bayesian}, so stand to improve the relaxation
significantly.

The \cpb solver~\cite{cpbayes} models BNSL as a constraint program. 
The CP model has a parent set
variable for each random variable, whose domain is the set of possible
parent sets, as well as order variables, which give a total order of
the variables that agrees with the partial order implied by the
DAG. The objective is the same as \eqref{eq:ilp-obj}.  It includes
channelling constraints between the set of variables and various
symmetry breaking and dominance constraints. It computes a lower bound
using two separate mechanisms: a component caching scheme and a
pattern database that is computed before search and holds the optimal
graphs for all orderings of partitions of the variables.
Acyclicity is enforced using a global constraint with a bespoke
propagator. The main routine of the propagator is \acycchecker
(Algorithm~\ref{alg:Acyc-Checker}), which returns an
order of all variables if the
current set of domains of the parent set variables may produce an
acyclic graph, or a partially completed order if
the constraint is unsatisfiable.
This algorithm is based on Property~\ref{prop:cluster}.

\begin{algorithm}[t]
  \DontPrintSemicolon
  \SetAlgoNoLine
  \SetKwFunction{acycchecker}{acycChecker}
  \SetKw{brk}{break}
  \acycchecker(\allvars, D) \;
  $order \gets \{\}$\;
  $changes \gets true$ \;
  \While{$changes$} {
    $changes \gets false$ \;
    \ForEach{$v \in \allvars \setminus order$} {
      \If{$\exists S \in D(v)$ s.t. $(S \cap V) \subseteq order$} { \lnl{alg:ln:witness}   $order \gets order + v$ \; 
        $changes \gets true$ \;
      }
    }
  }
  \Return {$order$}
\caption{Acyclicity Checker}
\label{alg:Acyc-Checker}
\end{algorithm}

Briefly, the algorithm takes the 
domains of the parent set variables
as input and greedily constructs an ordering of the variables, such
that if variable $v$ 
is later in the order than $v'$, then
$v \notin parents(v')$\footnote{We treat $order$
  as both a sequence and a set, as appropriate.}.
It does so by trying to pick a parent
set $S$ for an as yet unordered vertex such that $S$ is entirely
contained in the set of previously ordered vertices%
\footnote{When propagating the acyclicity constraint it always holds
  that $a \cap \allvars = a$, so this statement is true. In
  section~\ref{sec:minimisation}, we use the algorithm in a setting
  where this is not always the case.}.
If all assignments yield cyclic
graphs, it will reach a point where all remaining vertices are in a
violated cluster in all possible graphs, and it will return a
partially constructed order. If there exists an assignment that gives
an acyclic graph, it will be possible by property~\ref{prop:cluster} to
select from a variable in $\allvars \setminus order$ a parent set which
does not intersect $\allvars \setminus order$, hence is a subset of
$order$. The value $S$ chosen for each variable in
line~\ref{alg:ln:witness} also gives a witness of such an acyclic
graph.

An immediate connection between the \gob and \cpb models is that
the ILP variables $x_{v,S}, \forall S \in PS(v)$ are the direct 
encoding~\cite{Walsh00} of the parent set variables of the CP model. Therefore,
we use them interchangeably, i.e., we can refer to the value $S$ in $D(v)$ as 
$x_{v,S}$.

\section{Restricted Cluster Detection}\label{sec:Cluster}

One of the issues hampering the performance of \cpb is that it
computes relatively poor lower bounds at deeper levels of the search
tree. Intuitively, as the parent set variable domains 
get reduced by removing values that are inconsistent with
the current ordering, the lower
bound computation discards more information about the current state of
the problem.
%, making it weaker. 
We address this by
adapting the branch-and-cut approach of \gob. However, instead of
finding all violated cluster inequalities that may improve the LP
lower bound, we only identify a subset of them.

Consider the linear relaxation of the ILP
\eqref{eq:ilp-obj}--~\eqref{eq:ilp-integrality}, restricted to a
subset $\currentclusters$ of all valid cluster inequalities, i.e.,
with equation \eqref{eq:ilp-integrality} replaced by
$0 \leq x_{v,S} \leq 1 \forall v \in \allvars, S \in PS(v)$ and with
equation \eqref{eq:ilp-cluster} restricted only to clusters in
$\currentclusters$. We denote this $LP_{\currentclusters}$. We exploit
the following property of this LP.

\begin{theorem}
  \label{thm:rcclusters}
  Let $\yhat$ be a dual feasible solution of $LP_{\currentclusters}$
  with dual objective $o$. Then, if $C$ is a cluster such that
  $C \notin \currentclusters$ and the reduced cost $rc$ of all variables
  $\varsof{C}$ is greater than 0, there exists a dual feasible
  solution $\yhat$ of $LP_{\currentclusters \cup C}$ with dual
  objective $o' \geq o + \minrc{C}$ where 
  $\minrc C = \min_{x \in \varsof{C}} rc_{\yhat}(x)$.
\end{theorem}

\begin{proof}
  The
  only difference from $LP_{\currentclusters}$ to
  $LP_{\currentclusters \cup C}$ is the extra constraint $\ineq{C}$ in
  the primal and corresponding dual variable $y_C$. In the dual, $y_C$
  only appears in the dual constraints of the variables
  $\varsof{C}$ and in the objective, always with coefficient 1.
  Under the feasible dual solution
  $\yhat \cup \{y_C = 0\}$, these constraints have slack at least $\minrc{C}$,
  by the definition of reduced cost. Therefore, we can set
  $\yhat = \yhat \cup \{y_C = \minrc{C}\}$, which remains feasible and has
  objective $o' = o + \minrc{C}$, as required.
\end{proof}

Theorem~\ref{thm:rcclusters} gives a class of cluster cuts,
which we call \rcclusters, for reduced-cost clusters,
guaranteed to improve the lower bound. 
Importantly, this requires only 
a feasible, perhaps sub-optimal,
solution.

\begin{table}[ht]
\centering
\resizebox{0.5\columnwidth}{!}{%
\begin{tabular}{clr}
Variable           & Domain Value & Cost \\ \hline
0                  & $\{2\}$      & 0    \\ \hline
\multirow{2}{*}{1} & $\{2, 4\}$   & 0    \\ 
                   & $\{ \}$      & 6    \\ \hline
\multirow{2}{*}{2} & $\{1, 3\}$   & 0    \\ 
                   & $\{\}$       & 10   \\ \hline
\multirow{2}{*}{3} & $\{0\}$      & 0    \\ 
                   & $\{\}$       & 5    \\ \hline
\multirow{4}{*}{4} & $\{2, 3\}$   & 0    \\ 
                   & $\{3\}$      & 1    \\
                   & $\{2\}$      & 2    \\
                   & $\{\}$       & 3    \\ \hline
\end{tabular}%
}
\caption{BNSL instance used as running example.}
\label{tab:running_example}
\end{table}

\begin{example}[Running example]
  Consider a BNSL instance with domains as shown in
  Table~\ref{tab:running_example} and let
  $\currentclusters = \emptyset$. Then, $\yhat = 0$ leaves the reduced
  cost of every variable to exactly its primal objective coefficient. The
  corresponding \xhat assigns 1 to variables with reduced cost 0 and 0
  to everything else. These are both optimal solutions, with cost 0
  and \xhat is integral, so it is also a solution of the corresponding
  ILP. However, it is not a solution of the BNSL, as it contains
  several cycles, including $C = \{0, 2, 3\}$. The cluster inequality
  $\ineq{C}$ is violated in the primal and allows the dual bound to be
  increased.
\end{example}

We consider the problem of discovering \rcclusters within the CP model
of \cpb. First, we introduce the notation $LP_{\currentclusters}(D)$
which is $LP_{\currentclusters}$ with the additional constraint
$x_{v, S} = 0$ for each $S \notin D(v)$.
Conversely, $\rcD{D}{\currentclusters}{\yhat}$ is the set of
domains minus values whose corresponding variable 
in $LP_{\currentclusters}(D)$
has non-zero
reduced cost under $\yhat$, 
i.e., $\rcD{D}{\currentclusters}{\yhat} = D'$ where
$D'(v) = \{S \mid S \in D(v) \land rc_{\yhat}(x_{v,S}) = 0 \}$.
%where
%$rc(x_{v,S})$ is the reduced cost of $x_{v,S}$ in
%$LP_{\currentclusters}(D)$ under $\yhat$. 
With this notation,
for values $S \notin D(v)$, 
%i.e., $S$ is pruned in $D(v)$, 
$x_{v,S}=1$ is
infeasible in $LP_{\currentclusters}(D)$, hence effectively
$rc_{\yhat}(x_{v,S}) = \infty$.

\begin{theorem}
  \label{thm:boolofp}
  Given a collection of clusters $\currentclusters$, a set of domains
  $D$ and $\yhat$, a feasible dual solution of
  $LP_{\currentclusters}(D)$, there exists an \rccluster
  $C \notin \currentclusters$ if and only if
  $\rcD{D}{\currentclusters}{\yhat}$ does not admit an acyclic
  assignment.
\end{theorem}

\begin{proof}
$(\Rightarrow)$ Let $C$ be such a cluster. Since for all $x_{v,S} \in \varsof{C}$,
none of these are in $\rcD{D}{\currentclusters}{\yhat}$, so
$\ineq{C}$ is violated and hence there is no acyclic assignment.

$(\Leftarrow)$ Consider once again \acycchecker, in
Algorithm~\ref{alg:Acyc-Checker}. When it fails to find a witness of
acyclicity, it has reached a point where
$order \subsetneq \allvars$ and for the remaining variables 
$C = \allvars \setminus order$, all
allowed parent sets intersect $C$. So if \acycchecker is called with
$\rcD{D}{\currentclusters}{\yhat}$, all values in $\varsof{C}$
have reduced cost greater than 0, so $C$ is an \rccluster.
\end{proof}

Theorem~\ref{thm:boolofp} shows that detecting unsatisfiability of
$\rcD{D}{\currentclusters}{\yhat}$ is enough to find an \rccluster.
Its proof also gives a way to extract
such a cluster from \acycchecker.

\begin{algorithm}[t]
  \DontPrintSemicolon
  \SetAlgoNoLine
  \SetKwFunction{rcclusterlb}{lowerBoundRC}
  \SetKw{brk}{break}
  \rcclusterlb(\allvars, D, $\currentclusters$) \;
  $\yhat \gets DualSolve(LP_{\currentclusters}(D))$ \;
  \While{True} {
    \lnl{alg:rcclusterlb:bigcluster} $C \gets \allvars \setminus \acycchecker(\allvars, \rcD{D}{\currentclusters}{\yhat})$ \;
    \lnl{alg:rcclusterlb:terminate} \If{$C = \emptyset$}{ 
       \Return $\tuple{cost(\yhat), \currentclusters}$
    }
    $C \gets minimise(C)$ \;
    $\currentclusters \gets \currentclusters \cup \{ C \}$ \;
    $\yhat \gets DualImprove(\yhat, LP_{\currentclusters}(D), C)$ \;
  }
\caption{Lower bound computation with \rcclusters}
\label{alg:rcclusterlb}
\end{algorithm}

Algorithm~\ref{alg:rcclusterlb} shows how
theorems~\ref{thm:rcclusters} and~\ref{thm:boolofp} can be used to
compute a lower bound. It is given the current set of domains and a
set of clusters as input. It first solves the dual of
$LP_{\currentclusters}(D)$, potentially suboptimally. Then, it uses
\acycchecker iteratively to determine whether there exists an
\rccluster $C$ under the current dual solution $\yhat$. If that cluster
is empty, there are no more \rcclusters, and it terminates
and returns a lower bound equal to the cost of $\yhat$ under
$LP_{\currentclusters}(D)$ and an updated pool of clusters. 
Otherwise, it
minimises $C$ (see section~\ref{sec:minimisation}),
adds it to the pool of
clusters and solves the updated LP. 
It does this by calling $DualImprove$, which solves
$LP_{\currentclusters}(D)$ exploiting 
the fact that only the cluster inequality $\ineq{C}$ has
been added.

\begin{example}
  Continuing our example, consider the behavior of \acycchecker with
  domains $\rcD{D}{\emptyset}{\yhat}$ after the initial dual solution
  $\yhat = 0$. Since the empty set has non-zero reduced cost for all
  variables, \acycchecker fails with $order = \{\}$, hence $C = V$. We
  postpone discussion of minimization for now, other than to observe
  that $C$ can be minimized to $C_1 = \{1,2\}$. We add $\ineq{C_1}$ to
  the primal LP and set the dual variable of $C_1$ to 6 in the new
  dual solution $\yhat_1$. The reduced costs of $x_{1,\{\}}$ and
  $x_{2, \{\}}$ are decreased by 6 and, importantly,
  $rc_{\yhat_1}(x_{1,\{\}}) = 0$. In the next iteration of \rcclusterlb,
  \acycchecker is invoked on $\rcD{D}{\{C_1\}}{\yhat_1}$ and returns
  the cluster $\{0,2,3,4\}$. This is minimized to $C_2 =
  \{0,2,3\}$. The parent sets in the domains of these variables that
  do not intersect $C_2$ are $x_{2, \{\}}$ and $x_{3, \{\}}$, so
  $\minrc{C_2} = 4$, so we add $\ineq{C_2}$ to the primal and we set
  the dual variable of $C_2$ to 4 in $\yhat_2$. This brings the dual
  objective to 10. The reduced cost of $x_{2, \{\}}$ is 0, so in the
  next iteration \acycchecker runs on $\rcD{D}{\{C_1, C_2\}}{\yhat_2}$
  and succeeds with the order $\{2,0,3,4,1\}$, so the lower bound
  cannot be improved further. This also happens to be the cost of the
  optimal structure.
\end{example}

\begin{theorem}
  \label{thm:lbtermination}
  Algorithm~\ref{alg:rcclusterlb} terminates but is not confluent.
\end{theorem}

\begin{proof}
  It terminates because there is a finite number of cluster
  inequalities and each iteration generates one. In the extreme, all
  cluster inequalities are in $\currentclusters$ and the test at
  line~\ref{alg:rcclusterlb:terminate} succeeds, terminating the
  algorithm.

  To see that it is not confluent, consider an example with 3 clusters
  $C_1 = \{v_1, v_2\}, C_2 = \{v_2, v_3\}$ and $C_3 = \{v_3, v_4\}$
  and assume that the minimum reduced cost for each cluster 
  is unit and comes from
  $x_{2, \{4\}}$ and $x_{3, \{1\}}$,  i.e., the former value has
  minimum reduced cost for $C_1$ and $C_2$ and the latter for $C_2$
  and $C_3$. Then, if minimisation generates first $C_1$, the reduced
  cost of $x_{3, \{1\}}$ is unaffected by $DualImprove$, so it can
  then discover $C_3$, to get a lower bound of 2. On the other hand,
  if minimisation generates first $C_2$, the reduced costs of both
  $x_{2, \{4\}}$ and $x_{3, \{1\}}$ are decreased to 0 by
  $DualImprove$, so neither $C_1$ nor $C_3$ are \rcclusters under the
  new dual solution and the algorithm terminates with a lower bound of 1.
\end{proof}

\paragraph{Related Work.} The idea of performing propagation
on the subset of domains that have reduced cost 0 has been used
in the VAC algorithm for WCSPs~\cite{cooper2010soft}. Our method is more light weight,
as it only performs propagation on the acyclicity constraint, but
may give worse bounds.
The bound update mechanism 
in the proof of theorem~\ref{thm:rcclusters} is also
simpler than VAC and more akin to the ``disjoint
core phase'' in core-guided MaxSAT solvers~\cite{maxsat-survey}.

\subsection{Cluster Minimisation}\label{sec:minimisation}

\SetKwFunction{minimisecluster}{MinimiseCluster}

It is crucial for the quality of the lower bound produced by
Algorithm~\ref{alg:rcclusterlb} that the \rcclusters discovered by
\acycchecker are minimised, as the following example shows. 
Empirically, omitting minimisation rendered the lower bound
ineffective.

% \usepackage{multirow}
% \usepackage{graphicx}

\begin{example}
  \label{ex:minimisation}
  Suppose that we attempt to use \rcclusterlb without cluster
  minimization. Then, we use the cluster given by \acycchecker,
  $C_1 = \{0,1,2,3,4\}$. We have $\minrc{C_1} = 3$, given from the
  empty parent set value of all variables. This brings the reduced
  cost of $x_{4, \{\}}$ to 0. It then proceeds to find the cluster
  $C_2 = \{0,1,2,3\}$ with $\minrc{C_2} = 2$ and decrease the reduced
  cost of $x_{3,\{\}}$ to 0, then $C_3 = \{0,1,2\}$ with
  $\minrc{C_3} = 1$, which brings the reduced cost of $x_{1,\{\}}$ to
  0. At this point, \acycchecker succeeds with the order
  $\{4, 3, 1, 2, 0\}$ and \rcclusterlb returns a lower bound of 6,
  compared to 10 with minimization. The order produced by
  \acycchecker also disagrees with the optimum structure.
% Suppose the sets $C_1 = \{v_1, v_2\}$ and $C_2 = \{v_3, v_4\}$ 
% are \rcclusters
% and assume that 
% $\minrc {C_1} = \minrc {C_2} =c$.
% Then, their union $C_3 = \{v_1, v_2, v_3, v_4\}$ is also an \rccluster
% and $c = \minrc {C_3}$. Using $C_3$ 
% leaves the reduced cost of at least 1 variable
% in each of $\varsof{C_1}$ and $\varsof{C_2}$ at 0, so neither $C_1$ nor $C_2$
% are \rcclusters any more, giving a lower bound of $c$. If we instead use
% $C_1$ and $C_2$, we increase the lower bound by $2c$.
\end{example}

Therefore, when we get an \rccluster $C$ at line
\ref{alg:rcclusterlb:bigcluster} of algorithm~\ref{alg:rcclusterlb},
we want to extract a minimal \rccluster (with respect to set
inclusion) from $C$, i.e., a cluster $C' \subseteq C$, such that for
all $\emptyset \subset C'' \subset C'$, $C''$ is not a cluster.

Minimisation problems like this are handled with an
appropriate instantiation of QuickXPlain~\cite{Junker2004}.  These
algorithms find a minimal subset of constraints, not variables. We can
pose
this as a constraint set minimisation
problem by implicitly treating a variable as the constraint ``this
variable is assigned a value'' and treating acyclicity
as a hard constraint.

However, the property of being an \rccluster is not monotone. For
example, consider the variables $\{v_1,v_2, v_3,v_4\}$ and $\yhat$ such
that the domains restricted to values with 0 reduced cost are
$\{\{v_2\}\}, \{\{v_1\}\}, \{\{v_4\}\}, \{\{v_3\}\}$,
respectively. Then $\{v_1,v_2, v_3,v_4\}$, $\{v_1,v_2\}$ and
$\{v_3,v_4\}$ are \rcclusters. but $\{v_1,v_2,v_3\}$ is not because
the sole value in the domain of $v_3$ does not intersect
$\{v_1,v_2,v_3\}$. We instead minimise the set of variables that does not
admit an acyclic solution and hence \emph{contains} an \rccluster. 
A minimal unsatisfiable set that contains
a cluster is an \rccluster, so this allows us to
use %a minimisation algorithm for monotone predicates, such as
the variants of QuickXPlain. We focus on RobustXPlain, which
is called the deletion-based algorithm in SAT literature for
minimising unsatisfiable subsets~\cite{marquessilva-survey}.  The main
idea of the algorithm is to iteratively pick a variable and categorise it
as either appearing in all minimal subsets of $C$, in which case we
mark it as necessary,
or not, in which case we discard it.
To detect if a variable
appears in all minimal unsatisfiable subsets, we only have to test
if omitting this variable yields a set with no unsatisfiable subsets,
i.e., with no violated clusters. This is given in pseudocode in
Algorithm~\ref{alg::minimise-cluster}.
This exploits a subtle feature of \acycchecker as described in
Algorithm~\ref{alg:Acyc-Checker}: if it is called with a subset of
\allvars, it does not try to place the missing variables in the order
and allows parent sets to use these missing variables. Omitting
variables from the set given to \acycchecker acts as omitting the
constraint that these variables be assigned a value.
The complexity of
\minimisecluster is $O(n^3d)$, where $n = |V|$ and 
$d = \max_{v \in \allvars} |D(v)|$, a convention we adopt throughout.

\begin{algorithm}[t]
  \DontPrintSemicolon
  \minimisecluster(\allvars, D, C) \;
  $N = \emptyset$ \;
  \While{$C \neq \emptyset$}{
    Pick $c \in C$ \;
    $C \gets C \setminus \{c\}$ \;
    $C' \gets \allvars \setminus \acycchecker(N \cup C, D)$ \;
    \If{$C' = \emptyset$}{
      $N \gets N \cup \{c\}$ \;
    } \Else {
      $C \gets C' \setminus N$ \;
    }
  }
  \Return $N$ \;
\caption{Find a minimal \rccluster subset of $C$}
\label{alg::minimise-cluster}
\end{algorithm}

\section{Solving the Cluster LP}
\label{sec:lpsolve}

%%TERMINOLOGY: a cluster set is the set of random variables involved by a cluster cut
Solving a linear program is in polynomial time, so in principle $DualSolve$
can be implemented using any of the commercial or free software libraries
available for this. However, solving this LP using a general
LP solver is too expensive in this setting. As a data point, solving the
instance {\tt steel\_BIC} with our modified solver 
took 25,016 search nodes and 45 seconds of search, 
and generated $5,869$ \rcclusters.
Approximately 20\% of search time was spent solving the LP
using the greedy algorithm that we describe in this section.
CPLEX took around 70 seconds to solve $LP_{\currentclusters}$ with these
cluster inequalities once. While this data point is not proof
that solving the LP exactly is too expensive, it is a pretty strong indicator.
We have also not explored nearly linear time algorithms for solving positive 
LPs~\cite{lineartimeLP}.

Our greedy algorithm is
derived from theorem~\ref{thm:rcclusters}. Observe first that $LP_{\currentclusters}$
with $\currentclusters=\emptyset$, i.e., only with constraints~\eqref{eq:ilp-var}
has optimal dual solution $\yhat_0$ that assigns the dual variable $y_v$ of 
$\sum_{S \in PS(v)} x_{v,S} = 1$ to $\min_{S \in PS(v)} \sigma^v(S)$.
That leaves at least one of $x_{v,S}, S \in PS(v)$ with reduced cost 0 
for each $v \in \allvars$.
$DualSolve$ starts with $\yhat_0$ and then iterates over $\currentclusters$.
Given $\yhat_{i-1}$ and a cluster $C$, it sets $\yhat_i = \yhat_{i-1}$
if $C$ is not an \rccluster. Otherwise, it increases the lower bound
by $c = \minrc{C}$ and sets
$\yhat_{i} = \yhat_{i-1} \cup \{y_C = c\}$. It remains to specify the order
in which we traverse $\currentclusters$.

We sort clusters by increasing size $|C|$, breaking ties
by decreasing minimum cost of all original parent set
values in $\varsof{C}$. This favours finding 
non-overlapping cluster cuts with high minimum cost.
%
%Dynamically
%computing the current minimum reduced cost of all values in
%$\varsof{C}$ would make cluster comparison a linear time
%rather than constant time operation,
%which is why we prefer this less accurate static
%measure.
%
In section~\ref{sec:exp}, we give experimental evidence that
this computes better lower bounds.

$DualImprove$ can be implemented
by discarding previous information and calling
$DualSolve(LP_{\currentclusters}(D))$.
Instead, it 
uses the \rccluster $C$
to update the solution
without revisiting previous clusters.

In terms of implementation, we store $\varsof{C}$ for each cluster,
not $\ineq{C}$. During $DualSolve$, we maintain
the reduced costs of variables rather than the dual
solution, otherwise computing each reduced cost would require
iterating over all cluster inequalities that contain a
variable. Specifically, we maintain
$\Delta_{v,S} = \sigma^v(S) - rc_{\yhat}(x_{v,S})$.
In order to test whether a cluster $C$ is an \rccluster, we need to
compute $\minrc{C}$. To speed this up, we
associate with each stored cluster a {\em support pair} $(v,S)$ 
corresponding to the last minimum cost
found. If $rc_{\yhat}(v, S) = 0$, the cluster is 
not an \rccluster and is skipped.
Moreover, parent set domains are sorted by increasing score
$\sigma^v(S)$, so
$S \succ S' \iff \sigma^v(S) > \sigma^v(S')$. We also maintain
the maximum amount of cost transferred to the lower bound,
$\Delta^{max}_v = \max_{S \in D(v)} \Delta_{v,S}$ for every 
$v \in \allvars$. We stop iterating over $D(v)$ as soon as
$\sigma^v(S) - \Delta^{max}_v$ is greater than or equal to the current
minimum because 
$\forall S' \succ S, \sigma^v(S') -\Delta_{v,b} \geq
\sigma^v(S)-\Delta^{max}_v$.
In practice, on very large instances $97.6\%$ of 
unproductive clusters are detected by support
pairs and $8.6\%$ of the current domains are 
visited for the rest\footnote{
See the supplementary material for more.}.

To keep a bounded-memory cluster pool,
we discard frequently unproductive clusters. We
throw away large clusters with a productive ratio
$\frac{\#productive}{\#productive + \#unproductive}$ smaller than
$\frac{1}{1,000}$. Clusters of size 10 or less are always kept because
they are often more productive and their number
is bounded. %(see Figure \ref{fig-cluster-activity}). 

%On 41 large instances, the
%number of cuts found in preprocessing ranged from $328$ ({\tt
%  accidents.test}) to $4,011$ ({\tt kosarek.valid}) with a mean of
%$2,011.4$. At the end of the search (or 100-hour CPU-time limit), there were from $85$ %({\tt pumsb\_star.valid})
%to $163,710$ ({\tt baudio.ts}) clusters in the pool with a mean of $26,638.3$ clusters.

%\begin{figure}
%\centering
%    \includegraphics[width=.9\columnwidth]{15EX05_cuts_activity.pdf}
%	\caption{Cluster activity ($\#productive$) w.r.t. size for $9,018$ cuts kept in the cluster pool at the end of the search, after $14,840$ search nodes in  seconds on 15EX05 data set with 173 random variables.}
%	\label{fig-cluster-activity}
%\end{figure}

\section{GAC for the Acyclicity Constraint}\label{sec:GAC}

Previously, van Beek and Hoffmann\cite{cpbayes} showed that using
\acycchecker as a subroutine, one can construct a GAC propagator for
the acyclicity constraint by probing, i.e., detecting unsatisfiability
after assigning each individual value and pruning those values that
lead to unsatisfiability. \acycchecker is in $O(n^2d)$, so this
gives a GAC propagator in $O(n^3d^2)$.
We show here that we can
enforce GAC in time $O(n^3d)$,
a significant improvement given that $d$ is usually much larger
than $n$.

Suppose \acycchecker finds a witness of acyclicity and
returns the order $O = \{v_1, \ldots, v_n\}$.
Every parent set $S$ of a variable $v$ that is a subset of 
$\{v' \mid v' \prec_O v\}$ is supported by $O$. We call such values
consistent with $O$.
Consider now $S \in D(v_i)$ which is inconsistent with
$O$, therefore we have to probe
to see if it is supported. We know that
during the probe, nothing forces \acycchecker
to deviate from $\{v_1, \ldots, v_{i-1}\}$.
So in a successful probe, \acycchecker constructs a new order $O'$
which is identical to $O$ in the first $i-1$ positions and in which it
moves $v_i$ further down. 
Then all values
consistent with $O'$ are supported.
This suggests that instead of probing each value, we can
probe different orders.

\begin{algorithm}[t]
  \DontPrintSemicolon
  \SetAlgoNoLine
  \SetKwFunction{acycprop}{Acyclicity-GAC}
  \SetKw{brk}{break}
  \acycprop(V, D) \;
  $O \gets \acycchecker(V, D)$ \;
  \If{$O \subsetneq \allvars$} { \Return Failure }
  \ForEach{$v \in \allvars$}{
    $changes \gets true$ \;
    $i \gets O^{-1}(v)$ \;
    $prefix \gets \{ O_1, \ldots, O_{i-1} \}$ \;
    \lnl{alg:GACprop:push} \While{$changes$} { 
      $changes \gets false$\;
      \ForEach{$w \in O \setminus (prefix \cup \{v\})$} {
        \If{$\exists S \in D(w)$ s.t. $S \subseteq prefix$}{
          $prefix \gets prefix \cup \{w\}$ \; 
          $changes \gets true$ \;
        }
      }
    } 
    Prune $\{S \mid S \in D(v) \land S \nsubseteq prefix\}$ \;
  }
  \Return Success
\caption{GAC propagator for acyclicity}
\label{alg:GACprop}
\end{algorithm}

\acycprop, shown in Algorithm \ref{alg:GACprop}, exploits this insight.
It ensures first
that \acycchecker can produce a valid order $O$. 
For each variable $v$, it constructs a new order $O'$ from $O$ 
so that $v$ is as late as possible. It then prunes all
parent set values of $v$ that are inconsistent with $O'$.

%In the call in line \ref{lnl:checker-call}, this means that we use $O'_1, \ldots, O'_{i-1}$ as the initial order which we extend and we are required to place $X$ no earlier than position $i'+1$. We also use the notation $O^{-1}(X)$ to get the index of $X$ in the order $O$.

\begin{theorem}
\label{th:GAC}
  Algorithm~\ref{alg:GACprop} enforces GAC on the Acyclicity constraint in $O(n^3d)$.
\end{theorem}

\begin{proof}
Let $v \in \allvars$ and $S \in D(v)$. Let $O = \{O_1, \ldots, O_n\}$ and $Q = \{Q_1, \ldots, Q_n\}$ be two valid orders such that $O$ does not support $S$ whereas $Q$ does. 
It is enough
to show that we can compute from $O$ a new order $O'$ that supports $S$
by pushing $v$ towards the end.
Let $O_i = Q_j = v$ and let $O_p = \{O_1, \ldots, O_{(i-1)}\}$, $Q_p = \{Q_1, \ldots, Q_{(j-1)}\}$ and $O_s = \{O_{i+1}, \ldots, O_n\}$. 
%Given that $O$ does not support $a$ while $S$ does, we know that the two sets differ
%by at least one variable, i.e $S_p \setminus O_p \neq \emptyset$. Note that $O_p$ and $S_p$ form two DASGs (Directed Acyclic Sub-Graph), 
%meaning that the variables $O_p \cup S_p$ can form a DAG. 

Let $O'$ be the order $O_p$ followed by $Q_p$, followed by $v$, followed by $O_s$, 
keeping only the first occurrence of each variable when there are duplicates. $O'$
is a valid order: $O_p$ is witnessed by the assignment that witnesses $O$, $Q_p$ by
the assignment that witnesses $Q$, $v$ by $S$ (as in $Q$) and $O_s$ by the assignment that
witnesses $O$. It also supports $S$, as required.

Complexity is dominated by repeating $O(n)$ times
the loop at line
\ref{alg:GACprop:push}, which is a version
of \acycchecker so has complexity $O(n^2d)$ for a total
$O(n^3d)$.
\end{proof}

% \section{Exploiting Primal Solutions}

% A primal solution is evaluated each time the GAC algorithm terminates and produces a valid ordering or when the cluster reasoning stops with an empty cluster after finding a complete ordering with zero reduced-cost. If the cost of the primal solution is better than the previous best solution found so far, we update the current upper bound with this new cost. Every time a better solution is found, we also restart the depth-first branch-and-bound algorithm to the root node. Restarting \cpb is a valid operation if we remove from its learning mechanism (dominance test in~\cite{cpbayes}) all the prefix sets corresponding to search nodes interrupted by restart. Moreover, at each restart, we update the static value ordering used for ordering variables by exploiting the last solution found. It allows to guide the search towards better solutions more rapidly~\cite{DemirovicCP2018}.

%%We modified the input parameters of CPBayes in order to give as input an initial solution found by state-of-the-art BNSL local search methods\footnote{Replacing the existing hill-climbing local search in CPBayes.}. In practice, we run MINOBS genetic algorithm (Lee and van Beek CCAI 2017) with default parameters during 5 seconds (resp. 300 seconds) for small (resp. large) instances. The best ordering found by MINOBS is given as initial primal solution to our solver.

%$\sum_{v \in \Vbf} |ps(v)|$

\section{Experimental Results}
\label{sec:exp}

\begin{table*}[ht]
  \centering
  \small
  \begin{tabular}[t]{l l l  |r |r |r |r |r }
  \toprule
Instance	&  $|V|$	&  $\sum |ps(v)|$	&  \gob	&   \cpb	&  \cpbcut	& \cpbcut $\setminus$ GAC 	&  \cpbcut$^{chrono}$ \\
\hline                                       
%\endhead
%
  carpo100\_BIC	 &   60	 &   424	 &   \textbf{0.6}	 &   78.5	 (29.7) &   40.6	 (0.0) &   40.7	 (0.0) &   40.6 (0.0) \\
  alarm1000\_BIC	 &   37	 &   1003	 &   \textbf{1.2}	 &   204.2	 (172.9) &   27.8	 (0.7) &   28.8	 (1.5) &   29.9 (2.7) \\
  flag\_BDe	 &   29	 &   1325	 &   4.4	 &   19.0	 (18.1) &   \textbf{0.9}	 (0.1) &   0.9	 (0.1) &   1.3 (0.5) \\
  wdbc\_BIC	 &   31	 &   14614	 &   99.8	 &   629.8	 (576.6) &   \textbf{48.9}	 (1.6) &   49.1	 (1.7) &   50.3 (3.1) \\
  kdd.ts	 &   64	 &   43584	 &   \textbf{327.6}	 &   $\dagger$ &   1314.5	 (158.2) &   1405.4	 (239.5) &   1663.2  (512.4) \\
  steel\_BIC	 &   28	 &   93027	 &   $\dagger$	 &   1270.9	 (1218.9) &   \textbf{98.0}	 (49.2) &   99.2	 (50.1) &   130.0 (81.2) \\
  kdd.test	 &   64	 &   152873	 &   1521.7	 &   $\dagger$&   \textbf{1475.3}	 (120.6) &   1515.9	 (128.5) &   1492.4 (109.5) \\
  mushroom\_BDe	 &   23	 &   438186	 &   $\dagger$	 &   176.4	 (56.0) &   135.4	 (33.7) &   137.0	 (35.0) &   \textbf{133.7} (31.9) \\
\hline
  bnetflix.ts	 &   100	 &   446406	 &   $\dagger$	 &   \textbf{629.0}	 (431.4) &   1065.1	 (878.4) &   1111.4	 (931.0) &   1132.4 (936.3) \\
  plants.test	 &   69	 &   520148	 &   $\dagger$	 &   $\dagger$&   \textbf{18981.9}	 (17224.0) &   30791.2	 (29073.0) &   $\dagger$\\
  jester.ts	 &   100	 &   531961	 &   $\dagger$	 &   $\dagger$&   \textbf{10166.0}	 (9697.9) &   14915.9	 (14470.1) &   23877.6 (23325.7) \\
  accidents.ts	 &   111	 &   568160	 &   \textbf{1274.0}	 &   $\dagger$&   2238.7	 (904.5) &   2260.3	 (986.1) &   2221.1 (904.8) \\
  plants.valid	 &   69	 &   684141	 &   $\dagger$	 &   $\dagger$&   \textbf{12347.6}	 (8509.7) &   19853.1	 (15963.1) &   $\dagger$\\
  jester.test	 &   100	 &   770950	 &   $\dagger$	 &   $\dagger$&   \textbf{17637.8}	 (16979.2) &   21284.0	 (20661.9) &   $\dagger$\\
  bnetflix.test	 &   100	 &   1103968	 &   $\dagger$	 &   \textbf{3525.2}	 (3283.8) &   8197.7	 (7975.6) &   8057.3	 (7841.4) &   7915.0 (7686.3) \\
  bnetflix.valid	 &   100	 &   1325818	 &   $\dagger$	 &   \textbf{1456.6}	 (1097.0) &   9282.0	 (8950.3) &   10220.5	 (9898.4) &   9619.7 (9257.4) \\
  accidents.test	 &   111	 &   1425966	 &   4975.6	 &   $\dagger$&   \textbf{3661.7}	 (641.5) &   4170.1	 (1213.6) &   3805.2 (687.6) \\

% \begin{table*}[ht]
%   \centering
%   \small
%   \begin{tabular}[t]{l l l  |r |rr |rr |rr |rr }
%   \toprule
%                                        &                    &  \multicolumn{1}{c}{} &    \multicolumn{1}{c}{\gob}    & \multicolumn{2}{c}{\cpb} & \multicolumn{2}{c}{\cpbcut} & \multicolumn{2}{c}{\cpbcut $\setminus$ GAC} & \multicolumn{2}{c}{\cpbcut$^{chrono}$} \\ \hline
% Instance	&  $|V|$	&  $\sum |ps(v)|$	&  {Total}	&   {Search}	&  {Total}	&  {Search}	&  {Total}	&  {Search}	&  {Total}	&  {Search}	&  {Total} \\
                                       
% %\endhead
% %
% \input{table_results.dat}
\bottomrule
\end{tabular}
\caption{Comparison of \cpbcut against \gob and \cpb. Time limit for
  instances above the line is 1h, for the rest 10h. Instances are
  sorted by increasing total domain size. For variants of \cpb we
  report in parentheses time spent in search, after preprocessing
  finishes. $\dagger$ indicates a timeout.}
\label{tab:mytable}
\end{table*}

\subsection{Benchmark Description and Settings}
%% size: number of variables, total number of parentset values, maximum number of parents, (sample size if known)

The datasets come from the UCI Machine Learning Repository\footnote{\url{http://archive.ics.uci.edu/ml}}, the Bayesian Network Repository\footnote{\url{http://www.bnlearn.com/bnrepository}}, and the Bayesian Network Learning and Inference Package\footnote{\url{https://ipg.idsia.ch/software.php?id=132}}.
Local scores were computed from the datasets using B. Malone's code\footnote{\url{http://urlearning.org}}. BDeu and BIC scores were used for medium size instances (less than 64 variables) and only BIC score for large instances (above 64 variables). The maximum number of parents was limited to 5 for large instances (except for {\tt accidents.test} with maximum of 8), a high value that allows even learning complex structures~\cite{ScanagattaNIPS2015}. For example, {\tt jester.test} has 100 random variables, a sample size of $4,116$ and $770,950$ parent set values. For medium instances, no restriction was applied except for some BDeu scores (limit sets to 6 or 8 to complete the computation of the local scores within 24 hours of CPU-time~\cite{LeeBeek2017}). %%but {\em e.g.,} {\tt mushroom\_BDe} has 23 BN variables and $438,185$ parent set values with at most 14 candidate parents!

We have modified the C++ source of \cpb v1.1 by adding our lower bound mechanism and GAC propagator.
We call the resulting solver \cpbcut and have made it publicly available.
For the evaluation, we compare with \gob v1.6.3 using \scip v3.2.1 with \cplex v12.7.0.
All computations were performed on a single core of Intel Xeon E5-2680 v3 at 2.50 GHz and 256 GB of RAM with a 1-hour (resp. 10-hour) CPU time limit for medium (resp. large) size instances. We used default settings for \gob with no approximation in branch-and-cut ({\tt limits/gap = 0}). We used the same settings in \cpb and \cpbcut for their preprocessing phase (partition lower bound sizes $l_{min},l_{max}$ and local search number of restarts $r_{min},r_{max}$). We used two different settings depending on problem size $|V|$:  $l_{min}=20,l_{max}=26,r_{min}=50,r_{max}=500$ if $|V| \leq 64$, else $l_{min}=20,l_{max}=20,r_{min}=15,r_{max}=30$.

%\begin{figure}
%\centering
%    \includegraphics[width=.6\columnwidth]{node_speed.pdf}
%	\caption{CPU-time per search node on large instances.}
%	\label{fig-node-speed}
%\end{figure}

\subsection{Evaluation}
In Table \ref{tab:mytable} we present the runtime to solve each instance
to optimality with \gob, \cpb, and \cpbcut with default settings, without the GAC algorithm and without sorting the cluster pool (leaving clusters in chronological order, rather than the heuristic ordering presented in Section \ref{sec:lpsolve}). 
For the instances with $\|V\| \leq 64$ (resp. $>64)$, we had a time limit of 1 hour (resp. 10 hours). We exclude instances that were solved within the time limit by \gob and have a search time of less than 10 seconds for \cpb and all variants of \cpbcut. We also exclude 8
instances that were not solved to optimality by any method. 
This leaves us 17 instances to analyse here
out of 69 total.
More details are given in the supplemental material, available from the authors' web pages.

\paragraph{Comparison to \gob.}
\cpb was already proven to be competitive to \gob \cite{cpbayes}. Our results in Table \ref{tab:mytable} confirm this while showing that neither is clearly better. When it comes to our solver \cpbcut, for all the variants, all instances solved within the time limit by \gob are solved, unlike \cpb. On top of that, \cpbcut solves 9 more instances optimally.

\paragraph{Comparison to \cpb.}
We have made some low-level performance improvements in preprocessing of \cpb, so 
for a more fair comparison, we should compare only the search time, shown in parentheses. 
\cpbcut takes several orders of magnitude less search time to optimally solve most instances,
the only exception being the {\tt bnetflix} instances. \cpbcut also proved optimality 
for 8 more instances within the time limit.

\paragraph{Gain from GAC.} The overhead of GAC pays off as the instances get larger. While we do not see either a clear improvement nor a downgrade for the smaller instances,
search time for \cpbcut improves by up to $47\%$ for larger instances compared to \cpbcut $\setminus$ GAC.

\paragraph{Gain from Cluster Ordering.} We see that the ordering heuristic
improves the bounds computed by our greedy dual LP algorithm significantly.
Compared to not ordering the clusters, we see improved runtime throughout and
3 more instances solved to optimality.

%\begin{figure}
%\centering
%    \includegraphics[width=.6\columnwidth]{original_vs_ourDefault_NODES_SCATTER.png}
%	\caption{Number of search tree nodes for \cpb and \cpbcut.}
%	\label{fig-node-scatter}
%\end{figure}
\section{Conclusion}

We have presented a new set of inference techniques for BNSL using constraint programming,
centered around the expression of the acyclicity constraint.
These new techniques exploit and improve on
previous work on linear relaxations of the
acyclicity constraint and the associated
propagator. The resulting solver explores a different trade-off
on the axis of strength of inference versus speed, with \gob on
one extreme and \cpb on the other. We showed experimentally
that the trade-off we
achieve is a better fit than either extreme, as our solver \cpbcut
outperforms both \gob and \cpb.
The major obstacle towards better scalability to larger instances
is the fact that domain sizes grow exponentially with the number
of variables. This is to some degree unavoidable, so our
future work will focus on exploiting the structure of these
domains to improve performance.

\section*{Acknowledgements}
We thank the GenoToul (Toulouse, France) Bioinformatics platform for its %computational
support. This work has been partly funded by
the ``Agence nationale de la Recherche'' (ANR-16-CE40-0028 Demograph project and ANR-19-PIA3-0004 ANTI-DIL chair of Thomas Schiex).

\bibliographystyle{named}
\bibliography{ijcai21}
\end{document}